\documentclass[final,12pt]{clear2024} 

\usepackage[utf8]{inputenc} 
\usepackage[T1]{fontenc}    

\usepackage{booktabs}       
\usepackage{amsfonts}       
\usepackage{nicefrac}       
\usepackage{microtype}      
\usepackage{xcolor}         
\usepackage{wrapfig}
\usepackage{array}
\usepackage{bbm}
\usepackage{enumitem}
\usepackage{mathtools}

\usepackage{pkgs/nicematrix_old}
\usepackage{bm}

\usepackage{subcaption}

\usepackage{tikz}
\usetikzlibrary{positioning}
\usetikzlibrary{bayesnet}
\usetikzlibrary{fit}
\usetikzlibrary{shapes,arrows}
\usetikzlibrary{shapes.multipart}
\usetikzlibrary{arrows.meta}
\usetikzlibrary{fit}

\usepackage[customcolors]{hf-tikz}

\tikzset{style green/.style={
    set fill color=green!50!lime!60,
    set border color=white,
  },
  style black/.style={
    set fill color=black!10,
    set border color=white!0,
    rounded corners=2pt,
  },
  style orange/.style={
    set fill color=orange!80!red!60,
    set border color=white,
  },
  hor/.style={
    above left offset={-0.15,0.31},
    below right offset={0.15,-0.125},
    #1
  },
  ver/.style={
    above left offset={-0.1,0.3},
    below right offset={0.15,-0.15},
    #1
  }
}

\usepackage{breakcites}
\makeatletter

\renewcommand*{\@opargbegintheorem}[3]{\trivlist
  \item[\hskip \labelsep{\bfseries #1\ #2}] \textbf{(#3)}\ \itshape}
\makeatother

\def\*#1{\mathbf{#1}}
\def\Pa#1{\text{Pa}#1}

\newcommand{\ind}{\mathbbm{1}}
\newcommand{\Var}{\mathrm{Var}}

\newcommand\norm[1]{\lVert#1\rVert_{\Var}}

\makeatletter
\def\mathcolor#1#{\@mathcolor{#1}}
\def\@mathcolor#1#2#3{%
  \protect\leavevmode
  \begingroup
    \color#1{#2}#3%
  \endgroup
}
\makeatother

\usepackage{hyperref}
\usepackage[capitalise]{cleveref}

\newtheorem{assumption}[theorem]{Assumption}

\title[Identifying Linearly-Mixed Causal Representations from Multi-Node Interventions]{Identifying Linearly-Mixed Causal Representations\\ from Multi-Node Interventions}
\usepackage{times}

\clearauthor{%
 \Name{Simon Bing} \Email{\href{mailto:bing@tu-berlin.de?Subject=CLeaR2024-Multi Node CRL}{\color{black}{bing@tu-berlin.de}}}\\
 \addr Technische Universität Berlin
 \AND
 \Name{Urmi Ninad}\thanks{These authors contributed equally.} \Email{urmi.ninad@tu-berlin.de}\\
 \addr Technische Universität Berlin\\German Aerospace Center, Institute of Data Science
 \AND
 \Name{Jonas Wahl}$^*$\Email{wahl@tu-berlin.de}\\
 \addr Technische Universität Berlin\\German Aerospace Center, Institute of Data Science
 \AND
 \Name{Jakob Runge} \Email{runge@tu-berlin.de}\\
 \addr German Aerospace Center, Institute of Data Science\\Technische Universität Berlin%
}
\def\thefootnote{\arabic{footnote}}
\begin{document}

\maketitle

\begin{abstract}%
    The task of inferring high-level causal variables from low-level observations, commonly referred to as \emph{causal representation learning}, is fundamentally underconstrained. As such, recent works to address this problem focus on various assumptions that lead to identifiability of the underlying latent causal variables. A large corpus of these preceding approaches consider multi-environment data collected under different interventions on the causal model. What is common to almost all of these works is the restrictive assumption that in each environment, only a single variable is intervened on. In this work, we relax this assumption and provide a novel identifiability result for causal representation learning that allows for multiple variables to be targeted by an intervention within one environment. Our approach hinges on a general assumption on the coverage and diversity of interventions across environments, which also includes the shared assumption of single-node interventions of previous works. The main idea behind our approach is to exploit the trace that interventions leave on the variance of the ground truth causal variables and regularizing for a specific notion of sparsity with respect to this trace. In addition to and inspired by our theoretical contributions, we present a practical algorithm to learn causal representations from multi-node interventional data and provide empirical evidence that validates our identifiability results.
\end{abstract}


\section{Introduction}

Across the sciences, data are recorded in the form of low-level measurements of observable physical variables. However, the causal model underlying and robustly explaining the data may operate on high-level variables, sometimes called \emph{causally autonomous} or \emph{causally disentangled} variables. The task of learning such causal representations of data falls under the rubric of \emph{causal representation learning} \citep{scholkopf_toward_2021}. This problem has been shown to be fundamentally underconstrained \citep{locatello_challenging_2019}, leading to various approaches that employ inductive biases in order to provably identify the underlying latent causal variables.

Recent works that provide such identifiability guarantees either restrict the underlying structural causal model \citep{lachapelle_disentanglement_2022,buchholz_learning_2023,liang_causal_2023}, or the transformation mapping the causal variables to the observed data (also called the mixing function) 
\citep{ahuja_interventional_2023,zhang_identifiability_2023}, or both \citep{squires_linear_2023}. Many of these methods additionally assume the availability of interventional or counterfactual data collected across environments \citep{ahuja_interventional_2023,ahuja_multi-domain_2023,zhang_identifiability_2023,squires_linear_2023,buchholz_learning_2023,liang_causal_2023,von_kugelgen_nonparametric_2023,von_kugelgen_self-supervised_2021,brehmer_weakly_2022}, or use supervisory signals such as time structure \citep{hyvarinen_nonlinear_2017,halva_hidden_2020,yao_learning_2021}, sometimes in addition with knowledge of intervention targets \citep{lippe_causal_2022,lippe_citris_2022}.
What unites nearly all aforementioned methods that require the availability of interventional data---and guarantee full component-wise identifiability of the latent causal variables---is the restrictive assumption requiring interventions to be \emph{atomic} or \emph{single-node}, i.e., when the number of latent variables intervened on per environment is at most one.

\begin{figure}[t]
        \begin{tabular}{@{} m{.18\textwidth} @{} m{.273\textwidth} @{} m{.273\textwidth} @{} m{.273\textwidth} @{}}
            
            &
            \centering
            \resizebox{!}{3.2cm}{
            \tikz[latent/.append style={minimum size=0.85cm},obs/.append style={minimum size=0.85cm},det/.append style={minimum size=1.15cm}, wrap/.append style={inner sep=2pt}, on grid]{
                \node[latent](z_1){\(Z_1\)};
                \node[above=1.0cm of z_1, xshift=-1.5cm]{\(e = 1\)};
                \node[latent, left=1.2cm of z_1, yshift=-2.08cm](z_2){\(Z_2\)};
                \node[latent, right=1.2cm of z_1, yshift=-2.08cm](z_3){\(Z_3\)};
                \node[above=0.3cm of z_1, xshift=0.6cm](bolt){\includegraphics[width=0.55cm]{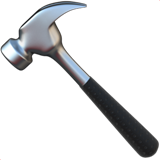}};
                \node[above=0.3cm of z_2, xshift=0.6cm](bolt){\includegraphics[width=0.55cm]{figures/hammer.png}};
                \edge[-{Stealth[length=2mm, width=1.5mm]}]{z_1}{z_3};
                \edge[-{Stealth[length=2mm, width=1.5mm]}]{z_2}{z_3};
            }
            }
            & 
            \centering
            \resizebox{!}{3.2cm}{
            \tikz[latent/.append style={minimum size=0.85cm},obs/.append style={minimum size=0.85cm},det/.append style={minimum size=1.15cm}, wrap/.append style={inner sep=2pt}, on grid]{
                \node[latent](z_1){\(Z_1\)};
                \node[above=1.0cm of z_1, xshift=-1.5cm]{\(e = 2\)};
                \node[latent, left=1.2cm of z_1, yshift=-2.08cm](z_2){\(Z_2\)};
                \node[latent, right=1.2cm of z_1, yshift=-2.08cm](z_3){\(Z_3\)};
                \node[above=0.3cm of z_1, xshift=0.6cm](bolt){\includegraphics[width=0.55cm]{figures/hammer.png}};
                \node[above=0.3cm of z_3, xshift=0.6cm](bolt){\includegraphics[width=0.55cm]{figures/hammer.png}};
                \edge[-{Stealth[length=2mm, width=1.5mm]}]{z_1}{z_2};
            }
            }
            &
            \centering
            \resizebox{!}{3.2cm}{
            \tikz[latent/.append style={minimum size=0.85cm},obs/.append style={minimum size=0.85cm},det/.append style={minimum size=1.15cm}, wrap/.append style={inner sep=2pt}, on grid]{
                \node[latent](z_1){\(Z_1\)};
                \node[above=1.0cm of z_1, xshift=-1.5cm]{\(e = 3\)};
                \node[latent, left=1.2cm of z_1, yshift=-2.08cm](z_2){\(Z_2\)};
                \node[latent, right=1.2cm of z_1, yshift=-2.08cm](z_3){\(Z_3\)};
                \node[above=0.3cm of z_2, xshift=0.6cm](bolt){\includegraphics[width=0.55cm]{figures/hammer.png}};
                \node[above=0.3cm of z_3, xshift=0.6cm](bolt){\includegraphics[width=0.55cm]{figures/hammer.png}};
            }
            }
        \end{tabular}
        \vspace{-0.5cm}
        \\
        \vspace{-0.7cm}
        \begin{tabular}{@{} m{.2\textwidth} @{} m{.266\textwidth} @{} m{.266\textwidth} @{} m{.266\textwidth} @{}}
            \((a)\) \begin{tabular}{c}
            Ground truth\\
            \(\*z\)
            \end{tabular}
            &
            \[\begin{bNiceMatrix}[first-row]
            \CodeBefore
            \Body
            z_1 & z_2 & z_3 \\
            1.0 & 1.0 & 1.01 \\
            1.0 & 1.0 & 0.98 \\
            \vdots & \vdots & \vdots \\
            1.0 & 1.0 & 1.12
            \end{bNiceMatrix}
             \]
            & 
            \[
            \begin{bNiceMatrix}[first-row]
            \CodeBefore
            \Body
            z_1 & z_2 & z_3 \\
            1.0 & 0.99 & 2.0 \\
            1.0 & 1.03 & 2.0 \\
            \vdots & \vdots & \vdots \\
            1.0 & 0.88 & 2.0
            \end{bNiceMatrix}
            \]
            &
            \[
            \begin{bNiceMatrix}[first-row]
            \CodeBefore
            \Body
            z_1 & z_2 & z_3 \\
            0.10 & 1.0 & 3.0 \\
            -0.11 & 1.0 & 3.0 \\
            \vdots & \vdots & \vdots \\
            0.32 & 1.0 & 3.0
            \end{bNiceMatrix}
            \]
        \end{tabular}
        \\
        \begin{tabular}{@{} m{.065\textwidth}@{} m{.135\textwidth} @{} m{.266\textwidth} @{} m{.266\textwidth} @{} m{.266\textwidth} @{}}
            \((b)\) 
            &
            \begin{tabular}{c}
            Mixture\\
            \(\tilde{\*z} := \*z \*L\)
            \end{tabular}
            &
            \[
            \begin{bNiceMatrix}[first-row]
            \CodeBefore
            \Body
            \tilde{z}_1 & \tilde{z}_2 & \tilde{z}_3 \\
            4.01 & 2.32 & 0.01 \\
            4.22 & 2.01 & -0.27 \\
            \vdots & \vdots & \vdots \\
            3.79 & 1.08 & 0.42
            \end{bNiceMatrix}
            \]
            & 
            \[
            \begin{bNiceMatrix}[first-row]
            \CodeBefore
            \Body
            \tilde{z}_1 & \tilde{z}_2 & \tilde{z}_3 \\
            3.09 & 2.09 & -0.02 \\
            3.72 & 2.23 & 0.26 \\
            \vdots & \vdots & \vdots \\
            4.43 & 1.95 & -0.12
            \end{bNiceMatrix}
            \]
            &
            \[
            \begin{bNiceMatrix}[first-row]
            \CodeBefore
            \Body
            \tilde{z}_1 & \tilde{z}_2 & \tilde{z}_3 \\
            3.79 & 2.63 & -2.13 \\
            3.51 & 2.11 & -2.32 \\
            \vdots & \vdots & \vdots \\
            4.07 & 1.72 & -1.96
            \end{bNiceMatrix}
            \]
        \end{tabular}
        \caption{Comparison of samples drawn from the SCM of \cref{example} under three different interventions between \textbf{(\(\bm{a}\))} the ground truth representation and \textbf{(\(\bm{b}\))} a mixed representation. Notice that the density of variables with nonzero variance is lower in each environment in the ground truth representation than in the mixed case. We exploit the principle that the ground truth is more sparse in terms of nonzero variance dimensions to achieve identifiability of causal representations using multi-node interventional data.}
        \label{fig:example}
    \end{figure}

In this work, we relax the restrictive single-node intervention assumption, and design a setup that guarantees the element-wise identifiability of latent causal variables by using data that is collected across multiple interventional environments, given the assumption that the mixing function is linear. However, importantly, we do not require these interventions to be atomic, only that they are \emph{hard} and sufficiently \emph{diverse} in a specific sense outlined below, see in particular \cref{ass:support}. Our assumption on the nature of interventions generalizes previous works in the sense that they allow for more general interventions, most notably those acting on multiple variables in one environment, while still allowing for the case of atomic interventions. Furthermore, we make no assumptions on the underlying structural causal model (SCM). 
To prove our main result, we introduce a specific notion of sparsity concerning the trace left in the data by interventions. As we detail in \cref{app:proof}, this general idea of using sparsity to achieve identifiability of representations is motivated by a similar assumption on the predictor of a  multi-task prediction problem of \citet{lachapelle_disentanglement_2022}.
However, our problem setting is not a prediction task, and therefore the nature of our sparsity assumption is fundamentally different. 
Our results also support the \emph{sparse mechanism shift} hypothesis, that posits that the distribution shifts of the underlying causal variables across environments must be sparse \citep{perry_causal_2022}.

Our \textbf{main contributions} can be summarized as follows:
\begin{enumerate}
    \item We formalize the notion that in the ground truth representation only a \emph{sparse} subset of variables is affected by an intervention, while this effect is dense for variables that have undergone mixing.
    \item Based on this, we prove identifiability for the setup of latent variables that have undergone linear mixing, given hard interventional data across environments. The SCM underlying the latent variables can be arbitrarily nonlinear with additive noise.
    Crucially, the interventional structure is not required to be atomic, but only sparse and diverse in the sense outlined in \cref{ass:support}.
    \item We present a proof-of-concept algorithm to recover the latent causal variables up to permutations and rescaling in the setting outlined above, alongside accompanying experiments.
\end{enumerate}

\section{Motivating Example}

We begin by presenting the following motivating example as an intuitive introduction to our approach to exploit a specific notion of sparsity to disentangle linear mixtures using multi-node interventional data.

\begin{example}\label{example}
    Consider the following structural causal model (SCM)
    \begin{equation*}
        \begin{tabular}{@{} m{.33\textwidth} @{} m{.33\textwidth} @{} m{.33\textwidth} @{}}
            \centering
            \resizebox{.2\textwidth}{!}{
            \tikz[latent/.append style={minimum size=0.85cm},obs/.append style={minimum size=0.85cm},det/.append style={minimum size=1.15cm}, wrap/.append style={inner sep=2pt}, on grid]{
                \node[latent](z_1){\(Z_1\)};
                \node[latent, left=1.2cm of z_1, yshift=-2.08cm](z_2){\(Z_2\)};
                \node[latent, right=1.2cm of z_1, yshift=-2.08cm](z_3){\(Z_3\)};
                \edge[-{Stealth[length=2mm, width=1.5mm]}]{z_1}{z_2,z_3};
                \edge[-{Stealth[length=2mm, width=1.5mm]}]{z_2}{z_3};
            }
            }
            & 
            \[
            \begin{aligned}
                 Z_1 &:= \eta_1 \\
                 Z_2 &:= Z_1 + \eta_2 \\
                 Z_3 &:= Z_1 + Z_2 + \eta_3
            \end{aligned}
            \]
            &
            \[\eta_j \sim \mathcal{N}(0,1),\]
        \end{tabular}
    \end{equation*}
    with \(d = 3\) variables \(\*Z = (Z_1, Z_2, Z_3)\) and independent noise terms $\eta_1,\eta_2,\eta_3$. Now, consider three environments, with the following corresponding interventions
    \begin{align*}
        I^1 &= \{\textnormal{do}(Z_1 = 1, Z_2 = 1)\},\\
        I^2 &= \{\textnormal{do}(Z_1 = 1, Z_3 = 2)\},\\
        I^3 &= \{\textnormal{do}(Z_2 = 1, Z_3 = 3)\}.
    \end{align*}
    We draw \(n\) samples from the above SCM in each environment. The resulting \(\*z^j \in \mathbb{R}^{n \times d}\) are shown in Fig.~\hyperref[fig:example]{1(a)}.
    
    Now, consider the same environments as above, but in the case where the variables in \(\*Z\) are mixed by the invertible matrix
    \begin{align*}
        \*L = \begin{bmatrix}
            1 & 1 & 1\\
            1 & -1 & 1\\
            1 & 1 & -1\\
        \end{bmatrix},
    \end{align*}
    denoting the mixed variables as \(\tilde{\*Z} = \*Z \*L\). Again drawing \(n\) samples for each environment---we recall that interventions act on the unmixed ground truth variables---the resulting \(\tilde{\*z}^j \in \mathbb{R}^{n \times d}\) is presented in Fig.~\hyperref[fig:example]{1(b)}.
   
    Notice that in the unmixed case, in each respective environment, the columns of \(\*z^j\) which have been intervened on take the same constant value for each sample and thus have zero variance, while those dimensions that are not targeted by interventions have nonzero variance and display stochasticity in the values they may take. Contrarily, in the mixed case, for this specific mixing matrix \(\*L\), in all environments, all columns \(\tilde{\*z}^j\) have nonzero variance.
\end{example}

The main observation to be made in the above example is the fact that mixing seems to increase the "density" of those dimensions of \(\tilde{\*Z}\) which are not constant and have nonzero variance, w.r.t the ground truth representation \(\*Z\).

Inspired by this observation, we propose to regularize the learned representation such that the effect of interventions is most sparse across environments. In the following, we make this specific notion of sparsity precise and show that this constraint indeed allows us to recover the latent causal factors up to permutation and rescaling (cf. \cref{def:disent}), under a general linear mixing matrix \(\*L\).

\section{Disentanglement from Multi-Node Interventions}\label{sec:disent}
In this section, we present the main results of this work on identifying causal representations from multi-node interventional data. The basic idea of our approach is to exploit the fact that interventions leave a specific trace in the data, which is most sparse in the ground truth representation. In the following, we make this notion of sparsity formal and prove that regularizing for it yields a representation of the latent variables that is equivalent to the ground truth.

\subsection{Problem Setting}\label{ssec:problem}
\paragraph{Notation.}
Scalar variables are denoted in normal face (\(x\)) and vector-valued variables in bold (\(\*x\)). Random variables are capitalized (\(Y\)), the values they take we write in lower case (\(y\)). Matrices are capitalized and bold (\(\*M\)) and will be introduced as such. We denote the sequence of integers from \(1\) to \(n\) with \([n]\).

\paragraph{Data Generating Process.}

Consider the random variables \(\*Z = (Z_1, \dots, Z_d)\). Assume \(\*Z \sim P\) where the unknown joint distribution \(P\) is induced by a structural causal model (SCM) defined over the random vector \(\*Z\). This SCM induces the factorization
\begin{align}
    P(\*Z) = \prod^{d}_{j=1} P(Z_j \mid Z_{\Pa_j}),
\end{align}
where \(\Pa_j \subset [d] \setminus \{j\}\) indicates the parents of variable \(Z_j\). Let \(\mathcal{G}\) denote the corresponding directed acyclic graph (DAG) of this SCM. For a detailed definition of SCMs, we refer the reader to \citet{pearl_causality_2009}. 

We make no parametric assumptions on the causal mechanisms of this SCM, i.e., the structural equation for each variable takes on the general form \(Z_j := f_j(Z_{\text{Pa}_j}, \eta_j)\), where \(\eta_j; j \in [d]\) are the exogenous, independent noise terms. We assume the distributions of the noises have nonzero variance.

Instead of measuring \(\*Z\) directly, we only have access to observations \(\tilde{\*Z} \in \mathbb{R}^m\), where \(\tilde{\*Z} = \*Z \*L\)
is generated from the causal variables by the injective linear map \(\*L: \mathbb{R}^d \rightarrow \mathbb{R}^m\) with \(m \geq d\).

Additionally, we assume to have access to observations \(\tilde{\*Z}\) from different environments, where each environment \(e \in \mathcal{E}\) corresponds to a setting where subsets \(T \subseteq [d]\) of the underlying, latent variables \(\*Z\) have been intervened upon. We assume that we are given a probability measure \(P_E\) on the set of environments \(\mathcal{E}\) with full support \(\mathcal{E}\), which describes the distribution of a random environment $E$. Note that in practice we will have distributions of the latent variables corresponding to different interventional environments, that we can then view as having been drawn from $P_E$, much in the way outlined in \citet{mooij_joint_2020}. We explicitly do not include the observational distribution, i.e., where no variables have been intervened on, in \(\mathcal{E}\). We consider do-interventions \citep{pearl_causality_2009}, sometimes called hard interventions, which replace the structural equations of intervened upon variables with constant values. We write
\begin{align}
        Z_j := a_j \quad \text{for } j \in T,
\end{align}
where \(\*a \in \mathbb{R}^{\left|T\right|}\). Each environment \(e\) is characterized by its corresponding intervention \(I^e := \{(T, \*a) \mid T \subseteq [d], \*a \in \mathbb{R}^{\left|T\right|}\}\) and we denote with \(P^e\) the distribution of the random vector \(\*Z^e\) induced by intervention \(I^e\). We stress that interventions are explicitly not assumed to only target single variables. We assume that the mixing function \(\*L\) remains unchanged across environments.

\paragraph{Objective.}
Our goal is to recover the latent variables \(\*Z\) from observations \(\tilde{\*Z}\). Specifically, we are interested in exploring how observing transformations of latent variables under different interventional environments \(e \in \mathcal{E}\) can be leveraged to learn the underlying causal variables.

Formally, recovering the latent variables \(\*Z\) from observations \(\tilde{\*Z}\) amounts to learning the inverse of the mixing function \(\*L\). 
We are only interested in latent variables that are equal to the ground truth up to permutation, element-wise rescaling and possible redundancies if the dimensionality of the learned representation is greater than that of the ground truth variables. Consequently, we define an equivalence class over such latents. 
While weaker notions exist \citep{hyvarinen_nonlinear_2017}, we refer to identifiability as being able to recover the latent variables up to this equivalence.
We call equivalent representations \emph{causally disentangled up to redundancies}.

\begin{definition}[Causal Disentanglement up to Redundancies]
\label[definition]{def:disent}
    A learned representation \(\hat{\*Z} \in \mathbb R^m\) is causally disentangled up to redundancies w.r.t. the ground truth representation \(\*Z \in \mathbb R^d\) if there exists a matrix \(\*L \in \mathbb R^{d \times m}\), where
    \begin{enumerate}
        \item each column of \(\*L\) contains at most one nonzero element,
        \item there are at least \(d\) columns in \(\*L\) with a nonzero element,
    \end{enumerate}
    s.t. \(\hat{\*Z} = \*Z \*L\) almost surely.
\end{definition}
This definition ensures that each learned variable \(\hat{Z}_i\) is either a scalar multiple of a ground truth variable \(Z_j\), or zero, and that all ground truth variables in \(\*Z\) appear as scalar multiples in \(\hat{\*Z}\) at least once. 
Notice for the special case where \(m=d\), the matrix \(\*L\) in the above definition can be decomposed into a diagonal invertible matrix \(\*D\) and a permutation matrix \(\*P\), i.e., \(\*L = \*D \*P\), and we recover the standard definition of causal disentanglement \citep{khemakhem_variational_2020,lachapelle_disentanglement_2022}.

\subsection{Assumptions}\label{ssec:ass}
Before presenting our main theorem, we introduce some additional objects and assumptions that are required to make our considerations precise.

First, let us formally define a quantity that allows us to measure the notion of sparsity we are interested in.
\begin{definition}[Variance density.]
    Let \(\*A \in \mathbb{R}^d\) be a random vector. Then
    \begin{align*}
        \norm{\*A} := \sum^{d}_{j = 1} \ind(\Var(A_j) \neq 0),
    \end{align*}
    where \(\ind(\cdot)\) is the indicator function.
\end{definition}

The next lemma states that there is no mixing matrix \(\*L\), such that in the resulting mixture, any variable has variance zero.

\begin{lemma}[Non-vanishing variance under mixing.]\label[lemma]{lem:var}
\label[assumption]{ass:var}
    For all invertible \(\*L \in \mathbb{R}^{d \times d}\),
    \begin{align*}
        \forall j \in [d], \Var((\*Z \*L)_j) \neq 0.
    \end{align*}
\end{lemma}
The proof is presented in \cref{app:ssec:lemmas}.

Further, we denote with \(S^e\) the set that contains the indices the elements of  \(\*Z^e\) that have nonzero variance in environment \(e\), that is
\begin{align*}
    S^e := \{j \in [d] \mid \Var(Z^e_j) \neq 0\}.
\end{align*}
Because we consider hard interventions that set the values of variables to constant values, \(S\) contains all variables that have \emph{not} been intervened upon in this environment. 
 
We consider again the environment distribution $P_E$ which we can factorize as
\begin{align*}
    P_{E} = \sum_{S \in \mathcal{P}([d])} p(S) P_{E \mid S},
\end{align*}
where \(\mathcal{P}([d])\) denotes all subsets of \([d]\), \(p(S) = P_E \left( \{e \in \mathcal{E} \ | \ S^e = S \} \right) \) and \(P_{E \mid S}\) is  the conditional distribution $P_{E \mid S} ( \ \cdot \ ) := P_E(\ \cdot \ | S) = \tfrac{P_E(\ \cdot \ \cap S)}{p(S)}$ if $p(S) \neq 0$ (and $0$ otherwise). We denote with \(\mathcal{S}\) the support of the distribution \(p(S)\), i.e., \(\mathcal{S}:=\{S \in \mathcal{P}([d]) \mid p(S) > 0\}\).
The next assumption concerns \(\mathcal{S}\) and in words states that the environments we observe must be diverse enough, such that for each latent dimension \(j\), if we consider all environments where we intervene on variable \(Z_j\), we do not always simultaneously intervene on another variable \(Z_i\).
\begin{assumption}[Sufficient coverage of interventions.]
\label[assumption]{ass:support}
    For all \(j \in [d]\)
    \begin{align*}
        \bigcup_{S \in \mathcal{S} \mid j \notin S} S = [d]\setminus\{j\}.
    \end{align*}
\end{assumption}
This assumption is arguably quite general, therefore some remarks on its implications are in order. First, notice that if a given variable \(Z_j\) is intervened on in \emph{all} environments \(\mathcal{E}\), the assumption cannot be fulfilled. Further, the set of environments that contain single-node interventions for all variables are subsumed and permitted by the above assumption, showing that it generalizes the most common assumption of preceding works.

A particularly interesting implication of this assumption lies in its connection to the notion of \emph{separating systems}, first introduced by \citet{katona_separating_1966}. This concept has proven useful for experimental design for causal discovery \citep{hyttinen_experiment_2013, shanmugam_learning_2015, kocaoglu_cost-optimal_2017, kocaoglu_experimental_2017} and also allows us to derive a lower bound on the number of environments required to fulfill the above assumption and therefore ultimately learn the underlying causal representation. Specifically, \cref{ass:support} is equivalent to the definition of \emph{strongly} separating systems in \citet[Definition~1]{kocaoglu_experimental_2017}, which by \citet[Lemma~2]{kocaoglu_experimental_2017} allows us to directly conclude that we can satisfy our central assumption with a minimum of \(2\lceil\log d \rceil\) environments. This stands in contrast to requiring \(d\) environments when only single-node interventions are considered and highlights that allowing multi-node interventions can lead to requiring less environments overall, when the interventions can be chosen freely. Note that \citet{lippe_intervention_2022} show a similar result for a lower bound on required environments for causal representation learning, however these are specific to their time series setting that additionally requires knowledge of intervention targets in each environment.

\subsection{Identifiability Result}
We are now ready to present our main identifiability result. Intuitively, it exploits the fact that in the ground truth representation, interventions set certain dimensions to a constant value, leaving only a sparse subset of all dimensions with nonzero variance. Conversely, under some linear mixing, generally \emph{all} latent dimensions will have nonzero variance terms. The idea is then to find a transformation that yields a representation that is as sparse as possible in the aforementioned sense, which we show results in recovering latent variables that are equivalent to the ground truth.

\begin{theorem}[Disentanglement via intervention sparsity.]
\label{th:disent}
    Assume the data generating process described in \cref{ssec:problem}. Let \(\*L \in \mathbb{R}^{d \times m}\) be an injective matrix and \(\hat{\*Z} = \*Z \*L\).
    Suppose \cref{ass:support} holds. Then, if 
    \begin{align*}
        \mathbb{E}_{P_E} \norm{\hat{\*Z}^E} \leq \mathbb{E}_{P_E} \norm{\*Z^E},
    \end{align*}
    \(\hat{\*Z}\) is causally disentangled up to redundancies w.r.t. \(\*Z\) (cf. \cref{def:disent}), where $\norm{\*Z^{\cdot} }$ is defined as the function $\norm{\*Z^{\cdot}}: (\mathcal{E},P_E) \to \mathbb{R}, \ e \mapsto \norm{\*Z^e}$.

    In particular, if \(m=d\), it holds that \(\*L = \*D \*P\), where \(\*D\) is a diagonal invertible matrix and \(\*P\) is a permutation matrix.
\end{theorem}
In words, Theorem \ref{th:disent} shows that under our assumptions, any representation that is as sparse as the ground truth representation (w.r.t. the variance density) must already be causally disentangled. This will motivate the optimization approach of the following section.
The proof of Theorem \ref{th:disent} is inspired by that of \citet[Theorem B.5]{lachapelle_synergies_2023}, and is presented in \cref{app:proof}.

\section{Experiments}
We illustrate the result presented in \cref{th:disent} by providing empirical evidence that we can recover the ground truth causal variables from observed data when our assumptions are met. We provide code to reproduce all of our experimental results\footnote{\url{https://github.com/simonbing/Multi-Node-CRL}}.

\subsection{Practical Algorithm}
Based on our theoretical results presented in \cref{sec:disent} we present an algorithm that implements the main principles of our theorem in terms of practical constraints, to recover the underlying causal representation from data.

Since the observed data \(\tilde{\*Z}\) is generated from the ground truth data \(\*Z\) according to \(\tilde{\*Z} = \*Z \*L\), our algorithm aims to learn a linear map \(\hat{\*L}: \mathbb{R}^m \rightarrow \mathbb{R}^d\), such that the resulting \(\hat{\*Z} = \*Z \*L \hat{\*L}\) is disentangled w.r.t the ground truth \(\*Z\) (cf. \cref{def:disent}). We learn \(\hat{\*L}\) via stochastic gradient descent by optimizing over the loss function
\begin{align}\label{eq:loss}
    \min_{\hat{\*L}} \mathcal{L}_{\textnormal{Var}} + \lambda_e \mathcal{L}_e + \lambda_m \mathcal{L}_m + \lambda_{\textnormal{diag}} \mathcal{L}_{\textnormal{diag}} + \lambda_{\textnormal{norm}} \mathcal{L}_{\textnormal{norm}},
\end{align}
where the individual loss terms are introduced in the following.

Since our theoretical results exploit assumptions on the sparsity of nonzero variance terms across environments, we introduce the matrix \(\*V \in \mathbb{R}^{e \times m}\), where the elements in row \(\*V_{i:}\) contain the variance of \(\tilde{\*Z}\) in environment \(e = i\); \(\*V_{i,j} := \Var(\tilde{Z}^i_j)\). In words, \(\*V\) contains the stacked variance terms of \(\tilde{\*Z}\) in each environment.

\paragraph{Total Variance Support.} The first term in \cref{eq:loss}, \(\mathcal{L}_{\textnormal{Var}}\), captures the assumption that the support of nonzero variance terms, i.e., \(\sum_{i,j}\ind(\*V_{i,j} \neq 0)\), is minimal across environments. Since the indicator function is not differentiable, we use the sigmoid function \(\sigma(x) := \frac{1}{1 + e^{-x}}\) and define
\begin{align*}
    \mathcal{L}_{\textnormal{Var}} := \sum^{e,m}_{i,j} \sigma(\*V_{i,j}).
\end{align*}

\paragraph{Per-Environment Variance Support.} The next term, \(\mathcal{L}_e\), enforces that there is no environment in which all variables are intervened upon at the same time. In terms of \(\*V\), this means that every row \(\*V_{i:}\) contains at least one nonzero entry, which implies that the sum across each row \(i\) is nonzero, i.e., \(\ind((\sum^m_j \*V_{i,j}) \neq 0)\). Again replacing the counting operation of the indicator function with a sum over sigmoid functions, we write
\begin{align*}
    \mathcal{L}_e := -\sum^e_i \sigma\bigr(\sum^m_j \*V_{i,j}\bigr),
\end{align*}
where the negative sign comes from the fact that we want this term to be nonzero.

\paragraph{Per-Dimension Variance Support.} This term reflects the deliberation that there is no dimension \(j \in [m]\) which is intervened upon in all environments. Analogous to \(\mathcal{L}_e\), in terms of \(\*V\) this means that we enforce that each column \(\*V_{:j}\) contains at least one nonzero entry and we write
\begin{align*}
    \mathcal{L}_m := -\sum^m_j \sigma\bigr(\sum^e_i \*V_{i,j}\bigr).
\end{align*}

\paragraph{Diagonal Sparsity.}
Consider the exemplary variance matrix for \(m = 3\) and with three environments
\begin{align*}
    \*V = 
    \begin{bmatrix}
        0 & 0 & *\\
        0 & 0 & *\\
        * & * & 0
    \end{bmatrix},
\end{align*}
where nonzero entries are denoted with \(*\). \(\*V\) satisfies the considerations represented by \(\mathcal{L}_e\) and \(\mathcal{L}_m\) (no nonzero rows or columns), but clearly is in violation of \cref{ass:support}, since \(Z_1\) and \(Z_2\) are always intervened on together. Considering the same number of environments, a variance matrix that satisfies all of the above constraints could, e.g., be
\begin{align*}
    \*V' = 
    \begin{bmatrix}
        0 & 0 & *\\
        * & 0 & 0\\
        0 & * & 0
    \end{bmatrix},
\end{align*}
where all but one variable is intervened on per environment. Let the \(k\)-th diagonal of the \((d \times d)\)-matrix \(\*A\) be defined as \(\textnormal{diag}_k (\*A) := \{\*A_{i,j}\, \forall \, i,j \in [d] \mid (i + k -1) \mod d = j \mod d\}\). In words, \(\textnormal{diag}_k (\*A)\) describes the diagonal elements of the matrix \(\*A\) that "wrap around", where any \(k > 1\) denotes the offset from the main diagonal. For a \(3 \times 3\)-matrix, consider the following example:
\begin{equation*}
    \begin{tabular}{@{} m{.5\textwidth} @{} m{.5\textwidth} @{}}
        \[\*A=
        \begin{bmatrix}
            a_{1,1} & \mathcolor{red}{a_{1,2}} & \mathcolor{blue}{a_{1,3}}\\
            \mathcolor{blue}{a_{2,1}} & a_{2,2} & \mathcolor{red}{a_{2,3}}\\
            \mathcolor{red}{a_{3,1}} & \mathcolor{blue}{a_{3,2}} & a_{3,3}
        \end{bmatrix},\]
        &
        \[
        \begin{aligned}
             \textnormal{diag}_1(\*A) = \{a_{1,1}, a_{2,2}, a_{3,3}\}, \\
             \textnormal{diag}_2(\*A) = \{\mathcolor{red}{a_{1,2}}, \mathcolor{red}{a_{2,3}}, \mathcolor{red}{a_{3,1}}\}, \\
             \textnormal{diag}_3(\*A) = \{\mathcolor{blue}{a_{1,3}}, \mathcolor{blue}{a_{2,1}}, \mathcolor{blue}{a_{3,2}}\}.
        \end{aligned}
        \]
    \end{tabular}
\end{equation*}

If we compare the diagonals of \(\*V\) and \(\*V'\) we notice that \(\*V'\)---which is aligned with our assumptions---has more diagonals that contain only zeros than \(\*V\). Based on this observation, we formulate the loss term \(\mathcal{L}_{\textnormal{diag}}\).

We operationalize the above argument that \(\*V\) should contain as many \(\*0\)-diagonals as possible by recalling the \(\ell_{2,1}\) norm for matrices, defined as \(\lVert A \rVert_{2,1} := \sum^d_{j=1} \lVert A_{:j} \rVert\), where \(\lVert \cdot \rVert\) denotes the Euclidean norm for vectors. Since regularizing with the \(\ell_{2,1}\) norm is known to promote column sparsity \citep{argyriou_convex_2008}, we define an analogous regularization to promote sparsity of diagonals:
\begin{align*}
    \mathcal{L}_{\textnormal{diag}} := \sum^m_{j=1} \lVert \textnormal{diag}_j (\*V) \rVert.
\end{align*}

\paragraph{Norm Regularization.} To prevent \(\*V\) from collapsing to all zeros, we enforce the Frobenius norm of \(\hat{\*L}\) to take a fixed value \(a \in \mathbb{R}_{>0}\). In practice, we choose \(a = 1\) and define
\begin{align*}
    \mathcal{L}_{\textnormal{norm}} := (\lVert \hat{\*L} \rVert - a)^2.
\end{align*}

\subsection{Synthetic Data Generation}\label{ssec:synth_data}
We generate synthetic data by sampling from randomly generated DAGs according to the Erdős–Rényi model \citep{erdos_random_1959}, where we change the number of nodes \(d\) and the probability of an edge being present in the graph \(p\) across experiments. Unless stated otherwise, we consider linear causal mechanisms \(f_j := \sum_{j \in \Pa_j} \alpha_j Z_j + \eta_j\) for all variables \(Z_j\), where we sample the coefficients \(\alpha_j\) independently from \(\mathcal{U}[-0.1, 1.0]\). The random noise \(\eta_j\) is independently sampled from \(\mathcal{N}(0, 0.1)\) for all \(j\). For all experiments, except the one in which we investigate the effect of varying sample size, we sample \(n = 10^5\) data points per environment.

To generate interventional data, we consider environments where for each latent variable \(Z_j\) we intervene on all \emph{other} variables \(Z_i, i \neq j\), satisfying \cref{ass:support}.

Since we only assume access to a mixture \(\tilde{\*Z} = \*Z \*L\) of the causal variables, after sampling according to the procedure detailed above, we randomly sample an invertible matrix \(\*L\) and apply this mixing. For the results shown in this section we sample one \(\*L\) and keep it fixed across runs, to focus on the randomness induced by different initializations. In practice we found that different mixing matrices \(\*L\) do not greatly influence the reported results, which we show with additional experiments in \cref{app:ssec:add_results}.

\begin{figure}
    \centering
    \includegraphics[width=\textwidth]{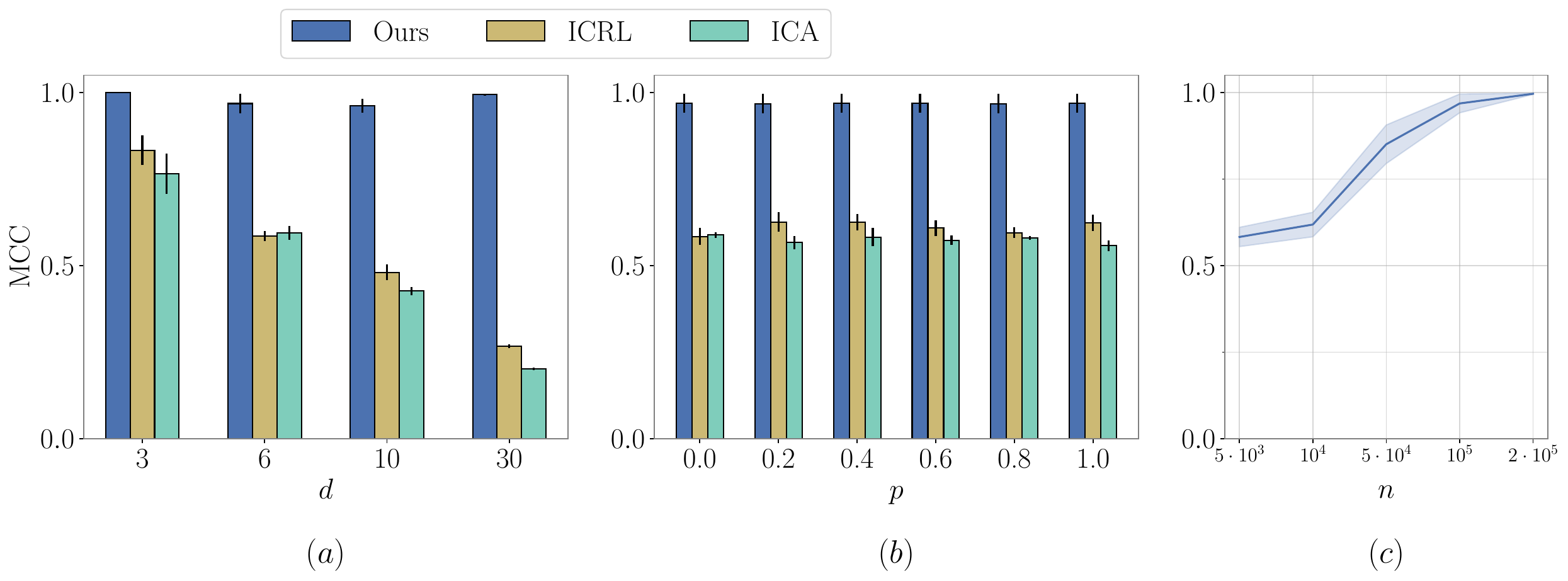}
    \caption{We report the mean MCC score across various experimental conditions, over five random seeds. Error bars or shaded regions indicate the standard error. \textbf{(\(\bm{a}\))} Our model performs well across all considered number of latent variables \(d\), even up to \(d=30\). \textbf{(\(\bm{b}\))} MCC across different probabilities of an edge being present \(p\). Our method achieves near-perfect score across all settings, indicating that we do not implicitly rely on assumptions on the density of the underlying graph. \textbf{(\(\bm{c}\))} MCC for different sample sizes \(n\). Performance increases with sample size and saturates at \(n=2\cdot10^5\). }
    \label{fig:results_overview}
\end{figure}

\subsection{Results}\label{ssec:results}
We provide the results of various experimental settings and compare against two exemplary baselines; linear independent component analysis (ICA) \citep{comon_independent_1994} and interventional causal representation learning (ICRL) \citep{ahuja_interventional_2023}. These are unfair comparisons, since both methods' assumptions are misaligned with our setting; ICA assumes independent, non-gaussian latent variables and ICRL assumes single-node interventions per environment. This comparison is not meant to show that our proposed method outperforms these approaches per se, but rather that our problem setting is not a contrived reformulation of a possibly simpler case, where ICA and ICRL are representative baselines for simpler settings. The main conclusion we can draw from our empirical findings is that our method recovers the ground truth latent variables in a more general problem setting than previous approaches.

To quantify how well we achieve our objective of recovering the equivalence class described in \cref{def:disent} we report the mean correlation coeeficient (MCC) \citep{hyvarinen_unsupervised_2016, khemakhem_variational_2020} (cf. \cref{app:ssec:metric}), which is precisely aligned with our notion of identifiability.

As suggested by our theoretical findings, across all experimental settings, our method achieves an almost perfect MCC score.

\paragraph{Effect of SCM Size.}
We investigate the effect of different underlying SCM sizes by considering \(d \in \{3, 6, 10, 30\}\). Shown in Fig.~\hyperref[fig:results_overview]{2(\(a\))}, we see that our model's performance does not noticeably fall off, even up to \(d = 30\). Both methods we compare against, ICA and ICRL, do not perform well, corroborating that our considered setting is indeed more general than their respective settings, and cannot be solved by either model.

\paragraph{SCM Density.}
In this experiment, we investigate the sensitivity of our approach to the density of the underlying ground truth causal graph. We vary the probability \(p\) of an edge being present in a graph with \(d=6\) nodes from 0 to 1, interpolating between a completely disconnected and a fully connected graph. The results are reported in Fig.~\hyperref[fig:results_overview]{2(\(b\))}. We see that the performance of our approach is equally high across all values of \(p\), underlining that we do not rely on assumptions regarding the sparsity of the underlying causal graph. Interestingly, our method can disentangle independent Gaussian latent variables (\(p = 0\)), which is precisely the case in which ICA fails.

\paragraph{Nonlinear SCMs.}
We consider two different SCMs, both with nonlinear mechanism functions \(f_j\) for each \(Z_j\). Both models consist of \(d=6\) variables, with an adjacency matrix randomly sampled according to the procedure described in \cref{ssec:synth_data}. In the first model, the mechanism functions are defined as \(f_j := \sum_{j \in \Pa_j} Z^2_j + \eta_j\) for all \(Z_j\), while the second model consists of more complex nonlinear mechanisms. See \cref{app:ssec:nonlin} for a detailed description of both models.
We report the MCC scores for both models in \cref{tab:nonlin_exp} and see that our method achieves similarly good results as on linear SCMs, underlining that we do not rely on parametric assumptions on the underlying causal model.

\paragraph{Number of Samples.}
We investigate the influence of sample size by increasing the number of data points \(n\) sampled per environment, for an SCM with \(d=6\). The results are visualized in Fig.~\hyperref[fig:results_overview]{2(\(c\))}, indicating that disentanglement is achieved after observing around \(10^5\) samples per environment, while the score saturates and perfect results are achieved after \(2\cdot10^5\) samples.

\begin{table}
    \centering
    \begin{tabular}{cc}
    \toprule
        Model & MCC\\
        \midrule
        Nonlin. SCM 1 & \(0.96 \pm 0.03\)\\
        Nonlin. SCM 2 & \(0.97 \pm 0.05\)\\
        \bottomrule
    \end{tabular}
    \caption{MCC (mean \(\pm\) standard error) over five different random seeds for two different nonlinear models. Our method performs equally well as for linear SCMs.}
    \label{tab:nonlin_exp}
\end{table}

\section{Related Work}
First exemplified by the impossibility result of nonlinear independent component analysis (ICA) \citep{hyvarinen_nonlinear_1999}, the general problem of representation learning is known to be heavily underconstrained \citep{locatello_challenging_2019}. Just as research in ICA has progressed by making assumptions explicit and subsequently exploiting these assumptions for inference \citep{hyvarinen_nonlinear_2017,halva_hidden_2020, gresele_independent_2021,morioka_connectivity-contrastive_2023}, recent works in causal representation learning works propose various assumptions to enable identifiability. 
One such inductive bias is to assume and exploit time structures, e.g., in \citet{lippe_citris_2022,lippe_causal_2022,lippe_biscuit_2023} by assuming knowledge of interventions targets or types, in \citet{yao_learning_2021,yao_temporally_2022} by the assumption of nonstationarity or in \citet{lachapelle_disentanglement_2022, lachapelle_nonparametric_2024} by imposing sparsity constraints on the underlying SCM.
A complementary line of works focuses on constraining the kind of mixing the latent variables have undergone, with a particular interest in the case of linear mixing. \citet{squires_linear_2023} show results for linear SCMs that undergo linear mixing, generalized by \citet{buchholz_learning_2023} to nonparametric SCMs, and \citet{varici_score-based_2023} provide a score-based approach for learning representations under linear mixtures.
One common assumption made across approaches is some kind of heterogeneity assumption on the available data, either induced by counterfactual pairs \citep{von_kugelgen_self-supervised_2021,brehmer_weakly_2022} or interventional data, e.g., hard do-interventions as in \citet{ahuja_interventional_2023} or soft interventions as in \citet{zhang_identifiability_2023}. What unites almost all works that include identifiability results based on some form of interventional (or counterfactual) data is the assumption that per environment, interventions only target a \emph{single} node \citep{brehmer_weakly_2022, squires_linear_2023, buchholz_learning_2023, ahuja_interventional_2023, zhang_identifiability_2023, varici_score-based_2023, liang_causal_2023, von_kugelgen_nonparametric_2023}. 
While there exist identifiability results with multi-node interventions in the time series setting they come with additional assumptions or limitations that we do not require. The results of \citet{lachapelle_disentanglement_2022, lachapelle_nonparametric_2024} only hold for an empty graph when only a single time step is considered and the results of \citet{lippe_causal_2022, lippe_citris_2022, lippe_biscuit_2023} require observing an auxiliary variable related to interventions at a given time step to show identifiability. \citet{ahuja_multi-domain_2023} and \citet{liang_causal_2023} also consider multi-node interventions, however they require additional assumptions and only guarantee identifiability up to blocks of variables, as opposed to our stronger component-wise identifiability.

An approach that utilizes a closely related assumption of sparsity to ours is presented by \citet{lachapelle_synergies_2023}. The authors consider a prediction problem in the context of representation learning, where the source of heterogeneity does not stem from interventions on the underlying SCM as in our case, but from multiple prediction tasks that share the same representation as potential predictors.
By assuming that only a sparse subset of all covariates are used per task, the latents can be identified up to permutation and rescaling. Analogously, we guarantee identifiability by assuming that a sparse subset of the ground truth latent variables have nonzero variance under interventions.

Another line of works frames causal representation learning as being closer to the setting of classical causal structure learning \citep{spirtes_causation_2001} with latent variables, commonly assuming that at least some nodes of the target graph are observed. While we asssume the observed variables to be deterministic functions of the latents, these works consider observational noise. Most approaches here rely on a variation of the \emph{pure children assumption}, stating that each observed variable has only a single latent parent \citep{silva_learning_2006}. \citet{cai_triad_2019} consider the case where observations are linear transformations of the latents and each latent has two pure children. \citet{xie_generalized_2020} generalize the preceding assumptions to allow latents with multiple children, formalized in their Generalized Independent Noise condition and in later work use this condition to learn hierarchical latent variable models \citep{xie_identification_2022}.

\section{Discussion}
In this work, we present a novel identifiability result for learning causal representations under linear mixing. Our main contribution lies in the generalization of a large body of preceding works 
that assume that the environments across which the data is collected correspond to single-node or atomic interventions on the latent variables, to the case where multi-node or non-atomic interventions are allowed as well. To enable identifiability from multi-node interventional data, we introduced a novel notion of sparsity regarding latent dimensions with nonzero variance terms. 

In addition to our theoretical contribution, we presented empirical evidence in the form of a proof-of-concept practical algorithm that recovers the latent variables of an underlying causal model in synthetic data settings where our assumptions are met.

\paragraph{Limitations.}
While linear mixing is also directly assumed in \citet{squires_linear_2023}, it is still a strong assumption. However, we argue that identifying causal variables from linear mixtures can be seen as a crucial module in a larger learning pipeline. Many complementary approaches result in recovering linear mixtures, such as assuming polynomial mixing functions \citep{ahuja_interventional_2023}, considering multi-task prediction problems \citep{lachapelle_synergies_2023}, learning of nonlinear causal effects with anchor variables \citep{saengkyongam_identifying_2023}, or a large class of problems where deep neural networks are used \citep{roeder_linear_2021}.

Our proposed practical implementation of our theory into a learning algorithm provides empirical evidence that our assumptions suffice for identifying causal variables using multi-node data. However, we only consider one class of interventions that satisfy our assumptions, while there are of course many more families of interventions that do so as well. While our algorithm is likely not the most general implementation of our theory, we stress that it merely stands as an empirical proof-of-concept and generalizing it to all possible cases where our assumptions are fulfilled was beyond the scope of this work.

\paragraph{Outlook.}
Since we consider hard (or do-) interventions in this work, an interesting avenue for future research would be to explore if this assumption can be relaxed to a more general family of interventions, similarly to how \citet{zhang_identifiability_2023} generalize the work of \citet{ahuja_interventional_2023}. Beyond that, exploring if our sparsity of nonzero variance dimensions principle is also applicable to nonlinear mixtures presents itself as a natural next step.

\acks{The authors thank Tom Hochsprung for fruitful discussions and comments, as well as the anonymous reviewers for their feedback and suggestions that helped improve the manuscript. This work received funding from the European Research Council (ERC) Starting Grant CausalEarth under the European Union’s Horizon 2020 research and innovation program (Grant Agreement No. 948112). S.B. received support from the German Academic Scholarship Foundation.}

\newpage
\bibliography{references}

\begin{thebibliography}{50}
\providecommand{\natexlab}[1]{#1}
\providecommand{\url}[1]{\texttt{#1}}
\expandafter\ifx\csname urlstyle\endcsname\relax
  \providecommand{\doi}[1]{doi: #1}\else
  \providecommand{\doi}{doi: \begingroup \urlstyle{rm}\Url}\fi

\bibitem[Ahuja et~al.(2023{\natexlab{a}})Ahuja, Mahajan, Wang, and Bengio]{ahuja_interventional_2023}
Kartik Ahuja, Divyat Mahajan, Yixin Wang, and Yoshua Bengio.
\newblock Interventional {Causal} {Representation} {Learning}.
\newblock In \emph{Proceedings of the 40th {International} {Conference} on {Machine} {Learning}}, pages 372--407, 2023{\natexlab{a}}.

\bibitem[Ahuja et~al.(2023{\natexlab{b}})Ahuja, Mansouri, and Wang]{ahuja_multi-domain_2023}
Kartik Ahuja, Amin Mansouri, and Yixin Wang.
\newblock Multi-{Domain} {Causal} {Representation} {Learning} via {Weak} {Distributional} {Invariances}.
\newblock \emph{arXiv preprint arXiv:2310.02854}, 2023{\natexlab{b}}.

\bibitem[Argyriou et~al.(2008)Argyriou, Evgeniou, and Pontil]{argyriou_convex_2008}
Andreas Argyriou, Theodoros Evgeniou, and Massimiliano Pontil.
\newblock Convex multi-task feature learning.
\newblock \emph{Machine Learning}, 73\penalty0 (3):\penalty0 243--272, 2008.

\bibitem[Brehmer et~al.(2022)Brehmer, de~Haan, Lippe, and Cohen]{brehmer_weakly_2022}
Johann Brehmer, Pim de~Haan, Phillip Lippe, and Taco~S. Cohen.
\newblock Weakly supervised causal representation learning.
\newblock In \emph{Advances in {Neural} {Information} {Processing} {Systems}}, 35, pages 38319--38331, 2022.

\bibitem[Buchholz et~al.(2023)Buchholz, Rajendran, Rosenfeld, Aragam, Schölkopf, and Ravikumar]{buchholz_learning_2023}
Simon Buchholz, Goutham Rajendran, Elan Rosenfeld, Bryon Aragam, Bernhard Schölkopf, and Pradeep Ravikumar.
\newblock Learning {Linear} {Causal} {Representations} from {Interventions} under {General} {Nonlinear} {Mixing}.
\newblock \emph{arXiv preprint arXiv:2306.02235}, 2023.

\bibitem[Cai et~al.(2019)Cai, Xie, Glymour, Hao, and Zhang]{cai_triad_2019}
Ruichu Cai, Feng Xie, Clark Glymour, Zhifeng Hao, and Kun Zhang.
\newblock Triad {Constraints} for {Learning} {Causal} {Structure} of {Latent} {Variables}.
\newblock In \emph{Advances in {Neural} {Information} {Processing} {Systems}}, 32, pages 12883--12892, 2019.

\bibitem[Comon(1994)]{comon_independent_1994}
Pierre Comon.
\newblock Independent component analysis, {A} new concept?
\newblock \emph{Signal Processing}, 36\penalty0 (3):\penalty0 287--314, 1994.

\bibitem[Erdős and Rényi(1959)]{erdos_random_1959}
Paul Erdős and Alfréd Rényi.
\newblock On {Random} {Graphs} {I}.
\newblock \emph{Publicationes Mathematicae Debrecen}, 6:\penalty0 290--297, 1959.

\bibitem[Gresele et~al.(2021)Gresele, Kügelgen, Stimper, Schölkopf, and Besserve]{gresele_independent_2021}
Luigi Gresele, Julius~Von Kügelgen, Vincent Stimper, Bernhard Schölkopf, and Michel Besserve.
\newblock Independent mechanism analysis, a new concept?
\newblock In \emph{Advances in {Neural} {Information} {Processing} {Systems}}, 34, pages 28233--28248, 2021.

\bibitem[Hyttinen et~al.(2013)Hyttinen, Eberhardt, and Hoyer]{hyttinen_experiment_2013}
Antti Hyttinen, Frederick Eberhardt, and Patrik~O. Hoyer.
\newblock Experiment {Selection} for {Causal} {Discovery}.
\newblock \emph{Journal of Machine Learning Research}, 14\penalty0 (93):\penalty0 3041--3071, 2013.

\bibitem[Hyvärinen and Oja(2000)]{hyvarinen_independent_2000}
A.~Hyvärinen and E.~Oja.
\newblock Independent component analysis: algorithms and applications.
\newblock \emph{Neural Networks}, 13\penalty0 (4):\penalty0 411--430, 2000.

\bibitem[Hyvärinen and Morioka(2016)]{hyvarinen_unsupervised_2016}
Aapo Hyvärinen and Hiroshi Morioka.
\newblock Unsupervised feature extraction by time-contrastive learning and nonlinear {ICA}.
\newblock In \emph{Advances in {Neural} {Information} {Processing} {Systems}}, 29, pages 3772--3780, 2016.

\bibitem[Hyvärinen and Morioka(2017)]{hyvarinen_nonlinear_2017}
Aapo Hyvärinen and Hiroshi Morioka.
\newblock Nonlinear {ICA} of {Temporally} {Dependent} {Stationary} {Sources}.
\newblock In \emph{Proceedings of the 20th {International} {Conference} on {Artificial} {Intelligence} and {Statistics}}, pages 460--469, 2017.

\bibitem[Hyvärinen and Pajunen(1999)]{hyvarinen_nonlinear_1999}
Aapo Hyvärinen and Petteri Pajunen.
\newblock Nonlinear independent component analysis: {Existence} and uniqueness results.
\newblock \emph{Neural Networks}, 12\penalty0 (3):\penalty0 429--439, 1999.

\bibitem[Hälvä and Hyvärinen(2020)]{halva_hidden_2020}
Hermanni Hälvä and Aapo Hyvärinen.
\newblock Hidden {Markov} {Nonlinear} {ICA}: {Unsupervised} {Learning} from {Nonstationary} {Time} {Series}.
\newblock In \emph{Proceedings of the 36th {Conference} on {Uncertainty} in {Artificial} {Intelligence}}, pages 939--948, 2020.

\bibitem[Katona(1966)]{katona_separating_1966}
Gyula Katona.
\newblock On separating systems of a finite set.
\newblock \emph{Journal of Combinatorial Theory}, 1\penalty0 (2):\penalty0 174--194, 1966.

\bibitem[Khemakhem et~al.(2020)Khemakhem, Kingma, Monti, and Hyvärinen]{khemakhem_variational_2020}
Ilyes Khemakhem, Diederik Kingma, Ricardo Monti, and Aapo Hyvärinen.
\newblock Variational {Autoencoders} and {Nonlinear} {ICA}: {A} {Unifying} {Framework}.
\newblock In \emph{Proceedings of {The} 23rd {International} {Conference} on {Artificial} {Intelligence} and {Statistics}}, pages 2207--2217, 2020.

\bibitem[Kocaoglu et~al.(2017{\natexlab{a}})Kocaoglu, Dimakis, and Vishwanath]{kocaoglu_cost-optimal_2017}
Murat Kocaoglu, Alex Dimakis, and Sriram Vishwanath.
\newblock Cost-{Optimal} {Learning} of {Causal} {Graphs}.
\newblock In \emph{Proceedings of the 34th {International} {Conference} on {Machine} {Learning}}, pages 1875--1884, 2017{\natexlab{a}}.

\bibitem[Kocaoglu et~al.(2017{\natexlab{b}})Kocaoglu, Shanmugam, and Bareinboim]{kocaoglu_experimental_2017}
Murat Kocaoglu, Karthikeyan Shanmugam, and Elias Bareinboim.
\newblock Experimental {Design} for {Learning} {Causal} {Graphs} with {Latent} {Variables}.
\newblock In \emph{Advances in {Neural} {Information} {Processing} {Systems}}, 30, 2017{\natexlab{b}}.

\bibitem[Lachapelle and Lacoste-Julien(2022)]{lachapelle_partial_2022}
Sébastien Lachapelle and Simon Lacoste-Julien.
\newblock Partial {Disentanglement} via {Mechanism} {Sparsity}.
\newblock \emph{arXiv preprint arXiv:2207.07732}, 2022.

\bibitem[Lachapelle et~al.(2022)Lachapelle, Rodriguez, Sharma, Everett, Le~Priol, Lacoste, and Lacoste-Julien]{lachapelle_disentanglement_2022}
Sébastien Lachapelle, Pau Rodriguez, Yash Sharma, Katie~E. Everett, Rémi Le~Priol, Alexandre Lacoste, and Simon Lacoste-Julien.
\newblock Disentanglement via mechanism sparsity regularization: {A} new principle for nonlinear {ICA}.
\newblock In \emph{Conference on {Causal} {Learning} and {Reasoning}}, pages 428--484, 2022.

\bibitem[Lachapelle et~al.(2023)Lachapelle, Deleu, Mahajan, Mitliagkas, Bengio, Lacoste-Julien, and Bertrand]{lachapelle_synergies_2023}
Sébastien Lachapelle, Tristan Deleu, Divyat Mahajan, Ioannis Mitliagkas, Yoshua Bengio, Simon Lacoste-Julien, and Quentin Bertrand.
\newblock Synergies between {Disentanglement} and {Sparsity}: {Generalization} and {Identifiability} in {Multi}-{Task} {Learning}.
\newblock In \emph{Proceedings of the 40th {International} {Conference} on {Machine} {Learning}}, pages 18171--18206, 2023.

\bibitem[Lachapelle et~al.(2024)Lachapelle, López, Sharma, Everett, Priol, Lacoste, and Lacoste-Julien]{lachapelle_nonparametric_2024}
Sébastien Lachapelle, Pau~Rodríguez López, Yash Sharma, Katie Everett, Rémi~Le Priol, Alexandre Lacoste, and Simon Lacoste-Julien.
\newblock Nonparametric {Partial} {Disentanglement} via {Mechanism} {Sparsity}: {Sparse} {Actions}, {Interventions} and {Sparse} {Temporal} {Dependencies}.
\newblock \emph{arXiv preprint arXiv:2401.04890}, 2024.

\bibitem[Liang et~al.(2023)Liang, Kekić, von Kügelgen, Buchholz, Besserve, Gresele, and Schölkopf]{liang_causal_2023}
Wendong Liang, Armin Kekić, Julius von Kügelgen, Simon Buchholz, Michel Besserve, Luigi Gresele, and Bernhard Schölkopf.
\newblock Causal {Component} {Analysis}.
\newblock \emph{arXiv preprint arXiv:2305.17225}, 2023.

\bibitem[Lippe et~al.(2022{\natexlab{a}})Lippe, Magliacane, Löwe, Asano, Cohen, and Gavves]{lippe_causal_2022}
Phillip Lippe, Sara Magliacane, Sindy Löwe, Yuki~M. Asano, Taco Cohen, and Efstratios Gavves.
\newblock Causal {Representation} {Learning} for {Instantaneous} and {Temporal} {Effects} in {Interactive} {Systems}.
\newblock In \emph{International {Conference} on {Learning} {Representations}}, 2022{\natexlab{a}}.

\bibitem[Lippe et~al.(2022{\natexlab{b}})Lippe, Magliacane, Löwe, Asano, Cohen, and Gavves]{lippe_citris_2022}
Phillip Lippe, Sara Magliacane, Sindy Löwe, Yuki~M. Asano, Taco Cohen, and Efstratios Gavves.
\newblock {CITRIS}: {Causal} {Identifiability} from {Temporal} {Intervened} {Sequences}.
\newblock In \emph{Proceedings of the 39th {International} {Conference} on {Machine} {Learning}}, pages 13557--13603, 2022{\natexlab{b}}.

\bibitem[Lippe et~al.(2022{\natexlab{c}})Lippe, Magliacane, Löwe, Asano, Cohen, and Gavves]{lippe_intervention_2022}
Phillip Lippe, Sara Magliacane, Sindy Löwe, Yuki~M. Asano, Taco Cohen, and Efstratios Gavves.
\newblock Intervention {Design} for {Causal} {Representation} {Learning}.
\newblock \emph{UAI 2022 Workshop on Causal Representation Learning}, 2022{\natexlab{c}}.

\bibitem[Lippe et~al.(2023)Lippe, Magliacane, Löwe, Asano, Cohen, and Gavves]{lippe_biscuit_2023}
Phillip Lippe, Sara Magliacane, Sindy Löwe, Yuki~M. Asano, Taco Cohen, and Efstratios Gavves.
\newblock {BISCUIT}: {Causal} {Representation} {Learning} from {Binary} {Interactions}.
\newblock In \emph{Proceedings of the 39th {Conference} on {Uncertainty} in {Artificial} {Intelligence}}, pages 1263--1273, 2023.

\bibitem[Locatello et~al.(2019)Locatello, Bauer, Lucic, Rätsch, Gelly, Schölkopf, and Bachem]{locatello_challenging_2019}
Francesco Locatello, Stefan Bauer, Mario Lucic, Gunnar Rätsch, Sylvain Gelly, Bernhard Schölkopf, and Olivier Bachem.
\newblock Challenging {Common} {Assumptions} in the {Unsupervised} {Learning} of {Disentangled} {Representations}.
\newblock In \emph{Proceedings of the 36th {International} {Conference} on {Machine} {Learning}}, pages 4114--4124, 2019.

\bibitem[Loshchilov and Hutter(2018)]{loshchilov_decoupled_2018}
Ilya Loshchilov and Frank Hutter.
\newblock Decoupled {Weight} {Decay} {Regularization}.
\newblock In \emph{International {Conference} on {Learning} {Representations}}, 2018.

\bibitem[Mooij et~al.(2020)Mooij, Magliacane, and Claassen]{mooij_joint_2020}
Joris~M. Mooij, Sara Magliacane, and Tom Claassen.
\newblock Joint causal inference from multiple contexts.
\newblock \emph{The Journal of Machine Learning Research}, 21\penalty0 (1):\penalty0 99:3919--99:4026, 2020.

\bibitem[Morioka and Hyvärinen(2023)]{morioka_connectivity-contrastive_2023}
Hiroshi Morioka and Aapo Hyvärinen.
\newblock Connectivity-contrastive learning: {Combining} causal discovery and representation learning for multimodal data.
\newblock In \emph{Proceedings of {The} 26th {International} {Conference} on {Artificial} {Intelligence} and {Statistics}}, pages 3399--3426, 2023.

\bibitem[Paszke et~al.(2019)Paszke, Gross, Massa, Lerer, Bradbury, Chanan, Killeen, Lin, Gimelshein, Antiga, Desmaison, Köpf, Yang, DeVito, Raison, Tejani, Chilamkurthy, Steiner, Fang, Bai, and Chintala]{paszke_pytorch_2019}
Adam Paszke, Sam Gross, Francisco Massa, Adam Lerer, James Bradbury, Gregory Chanan, Trevor Killeen, Zeming Lin, Natalia Gimelshein, Luca Antiga, Alban Desmaison, Andreas Köpf, Edward Yang, Zach DeVito, Martin Raison, Alykhan Tejani, Sasank Chilamkurthy, Benoit Steiner, Lu~Fang, Junjie Bai, and Soumith Chintala.
\newblock {PyTorch}: an imperative style, high-performance deep learning library.
\newblock In \emph{Advances in {Neural} {Information} {Processing} {Systems}}, 32, pages 8026--8037, 2019.

\bibitem[Pearl(2009)]{pearl_causality_2009}
Judea Pearl.
\newblock \emph{Causality}.
\newblock Cambridge University Press, 2009.

\bibitem[Perry et~al.(2022)Perry, von Kügelgen, and Schölkopf]{perry_causal_2022}
Ronan Perry, Julius von Kügelgen, and Bernhard Schölkopf.
\newblock Causal {Discovery} in {Heterogeneous} {Environments} {Under} the {Sparse} {Mechanism} {Shift} {Hypothesis}.
\newblock In \emph{Advances in {Neural} {Information} {Processing} {Systems}}, 35, pages 10904--10917, 2022.

\bibitem[Roeder et~al.(2021)Roeder, Metz, and Kingma]{roeder_linear_2021}
Geoffrey Roeder, Luke Metz, and Durk Kingma.
\newblock On {Linear} {Identifiability} of {Learned} {Representations}.
\newblock In \emph{Proceedings of the 38th {International} {Conference} on {Machine} {Learning}}, pages 9030--9039, 2021.

\bibitem[Saengkyongam et~al.(2023)Saengkyongam, Rosenfeld, Ravikumar, Pfister, and Peters]{saengkyongam_identifying_2023}
Sorawit Saengkyongam, Elan Rosenfeld, Pradeep Ravikumar, Niklas Pfister, and Jonas Peters.
\newblock Identifying {Representations} for {Intervention} {Extrapolation}.
\newblock \emph{arXiv preprint arXiv:2310.04295}, 2023.

\bibitem[Schölkopf et~al.(2021)Schölkopf, Locatello, Bauer, Ke, Kalchbrenner, Goyal, and Bengio]{scholkopf_toward_2021}
Bernhard Schölkopf, Francesco Locatello, Stefan Bauer, Nan~Rosemary Ke, Nal Kalchbrenner, Anirudh Goyal, and Yoshua Bengio.
\newblock Toward {Causal} {Representation} {Learning}.
\newblock \emph{Proceedings of the IEEE}, 109\penalty0 (5):\penalty0 612--634, 2021.

\bibitem[Shanmugam et~al.(2015)Shanmugam, Kocaoglu, Dimakis, and Vishwanath]{shanmugam_learning_2015}
Karthikeyan Shanmugam, Murat Kocaoglu, Alexandros~G. Dimakis, and Sriram Vishwanath.
\newblock Learning {Causal} {Graphs} with {Small} {Interventions}.
\newblock In \emph{Advances in {Neural} {Information} {Processing} {Systems}}, 28, pages 3195--3203, 2015.

\bibitem[Silva et~al.(2006)Silva, Scheine, Glymour, and Spirtes]{silva_learning_2006}
Ricardo Silva, Richard Scheine, Clark Glymour, and Peter Spirtes.
\newblock Learning the {Structure} of {Linear} {Latent} {Variable} {Models}.
\newblock \emph{Journal of Machine Learning Research}, 7\penalty0 (8):\penalty0 191--246, 2006.

\bibitem[Spirtes et~al.(2001)Spirtes, Glymour, and Scheines]{spirtes_causation_2001}
Peter Spirtes, Clark Glymour, and Richard Scheines.
\newblock \emph{Causation, {Prediction}, and {Search}}.
\newblock The MIT Press, 2001.

\bibitem[Squires et~al.(2023)Squires, Seigal, Bhate, and Uhler]{squires_linear_2023}
Chandler Squires, Anna Seigal, Salil~S. Bhate, and Caroline Uhler.
\newblock Linear {Causal} {Disentanglement} via {Interventions}.
\newblock In \emph{Proceedings of the 40th {International} {Conference} on {Machine} {Learning}}, pages 32540--32560, 2023.

\bibitem[Varıcı et~al.(2023)Varıcı, Acartürk, Shanmugam, Kumar, and Tajer]{varici_score-based_2023}
Burak Varıcı, Emre Acartürk, Karthikeyan Shanmugam, Abhishek Kumar, and Ali Tajer.
\newblock Score-based {Causal} {Representation} {Learning} with {Interventions}.
\newblock \emph{arXiv preprint arXiv:2301.08230}, 2023.

\bibitem[von Kügelgen et~al.(2021)von Kügelgen, Sharma, Gresele, Brendel, Schölkopf, Besserve, and Locatello]{von_kugelgen_self-supervised_2021}
Julius von Kügelgen, Yash Sharma, Luigi Gresele, Wieland Brendel, Bernhard Schölkopf, Michel Besserve, and Francesco Locatello.
\newblock Self-{Supervised} {Learning} with {Data} {Augmentations} {Provably} {Isolates} {Content} from {Style}.
\newblock In \emph{Advances in {Neural} {Information} {Processing} {Systems}}, 34, pages 16451--16467, 2021.

\bibitem[von Kügelgen et~al.(2023)von Kügelgen, Besserve, Liang, Gresele, Kekić, Bareinboim, Blei, and Schölkopf]{von_kugelgen_nonparametric_2023}
Julius von Kügelgen, Michel Besserve, Wendong Liang, Luigi Gresele, Armin Kekić, Elias Bareinboim, David~M. Blei, and Bernhard Schölkopf.
\newblock Nonparametric {Identifiability} of {Causal} {Representations} from {Unknown} {Interventions}.
\newblock \emph{arXiv preprint arXiv:2306.00542}, 2023.

\bibitem[Xie et~al.(2020)Xie, Cai, Huang, Glymour, Hao, and Zhang]{xie_generalized_2020}
Feng Xie, Ruichu Cai, Biwei Huang, Clark Glymour, Zeng Hao, and Kun Zhang.
\newblock Generalized independent noise condition for estimating latent variable causal graphs.
\newblock In \emph{Advances in {Neural} {Information} {Processing} {Systems}}, 33, pages 14891--14902, 2020.

\bibitem[Xie et~al.(2022)Xie, Huang, Chen, He, Geng, and Zhang]{xie_identification_2022}
Feng Xie, Biwei Huang, Zhengming Chen, Yangbo He, Zhi Geng, and Kun Zhang.
\newblock Identification of {Linear} {Non}-{Gaussian} {Latent} {Hierarchical} {Structure}.
\newblock In \emph{Proceedings of the 39th {International} {Conference} on {Machine} {Learning}}, pages 24370--24387, 2022.

\bibitem[Yao et~al.(2021)Yao, Sun, Ho, Sun, and Zhang]{yao_learning_2021}
Weiran Yao, Yuewen Sun, Alex Ho, Changyin Sun, and Kun Zhang.
\newblock Learning {Temporally} {Causal} {Latent} {Processes} from {General} {Temporal} {Data}.
\newblock In \emph{International {Conference} on {Learning} {Representations}}, 2021.

\bibitem[Yao et~al.(2022)Yao, Chen, and Zhang]{yao_temporally_2022}
Weiran Yao, Guangyi Chen, and Kun Zhang.
\newblock Temporally {Disentangled} {Representation} {Learning}.
\newblock In \emph{Advances in {Neural} {Information} {Processing} {Systems}}, 35, pages 26492--26503, 2022.

\bibitem[Zhang et~al.(2023)Zhang, Squires, Greenewald, Srivastava, Shanmugam, and Uhler]{zhang_identifiability_2023}
Jiaqi Zhang, Chandler Squires, Kristjan Greenewald, Akash Srivastava, Karthikeyan Shanmugam, and Caroline Uhler.
\newblock Identifiability {Guarantees} for {Causal} {Disentanglement} from {Soft} {Interventions}.
\newblock \emph{arXiv preprint arXiv:2307.06250}, 2023.

\end{thebibliography}

\newpage
\appendix

\section{Proofs}
\subsection{Auxiliary Lemmata}\label{app:ssec:lemmas}
In order to prove our main result in \cref{th:disent}, we require some additional lemmata, which we state and prove here.

\setcounter{theorem}{2}
\begin{lemma}[Non-vanishing variance under mixing.]\label[lemma]{app:lem:var}
\label[assumption]{ass:var}
    For all invertible \(\*L \in \mathbb{R}^{d \times d}\),
    \begin{align*}
        \forall j \in [d], \Var((\*Z \*L)_j) \neq 0.
    \end{align*}
\end{lemma}
\begin{proof}

    We allow for arbitrary SCMs with jointly independent noise terms, which are assumed to be non-degenerate in the sense that no variable \(Z_i\) is independent of its own noise term \(\eta_i\). In other words, no variable is deterministically determined by its parents.
    
    Let us label the $Z_i$'s such that they are topologically ordered, i.e., $Z_i$ is causally prior to $Z_j$ for $i<j$ and $Z_d$ is at bottom of the topological order. Each $Z_i$ can be expressed as a function of its parents and respective noise term, i.e.,
    \begin{equation}
        Z_i = f_i(Z_1, \ldots Z_{i-1}, \eta_i) \ ,
    \end{equation}
    where in this formulation we do not require functions $f_i$ to be minimal, that is $f_i$ is not required to depend on all of its input arguments. Additionally the noises $\eta_i$ are jointly independent and have non-vanishing variance. Solving the latent variables in terms of the independent noises $\eta$ yields, 
        \begin{equation}
        Z_i = g_i(\eta_1, \ldots ,\eta_{i}) \ ,
    \end{equation}
    where the $g_i$'s can be calculated by successively substituting structural equations of parents into those of children. \\
    Consider an invertible matrix $\*L \in \mathbb{R}^{d \times d}$ and an arbitrary index $j \in [d]$. Since $\*L$ is invertible, its $j$-th column must have nonzero entries which we will denote by $\alpha_1 = (\*L)_{i_1,j} ,\dots,\alpha_s = (\*L)_{i_s,j} \in \mathbb{R}\backslash\{0 \}$ where $i_1 < i_2< \dots < i_s$. Then, we can write 
    \begin{equation}\label{eq:solutionSCM}
        \tilde{Z}:= (\*Z \*L)_j = \sum\limits_{k=1}^s\alpha_k Z_{i_k}. 
    \end{equation}

    Now, assume \(\Var(\tilde{Z})=0\) from which we will derive a contradiction. Since the square root of the variance is a norm (the \(L^2\)-norm to be precise), \(\Var(\tilde{Z})=0\) implies that
    \begin{equation}
        \sum_{k=1}^s\alpha_k Z_{i_k} = \sum_{k=1}^s \alpha_k g_k(\eta_{i_1}, \ldots, \eta_{i_k}) = 0
    \end{equation}
    almost surely with respect to the noise distribution. Since all \(\alpha_k\)'s are nonzero we have
    \begin{equation}
        Z_{i_s} = -\frac{1}{\alpha_s} \sum_{k=1}^{s-1} \alpha_k g_k(\eta_{i_1}, \ldots, \eta_{i_k}).
    \end{equation}
    Hence, we have expressed \(Z_{i_s}\) as a function of the noise terms up to \(i_{s-1}\), which by the joint independence of noise terms implies \(Z_{i_s} \perp \eta_{i_s}\). This contradicts our non-degeneracy assumption stated at the beginning of the proof.
\end{proof}

Before presenting the proof of our main theorem, we require an additional technical lemma, adapted from \citet[Lemma B.1]{lachapelle_synergies_2023}. We restate it here.

\begin{lemma}[Invertible matrices contain a permutation, \citep{lachapelle_synergies_2023}]
\label[lemma]{lem:permutation}
    Let \(\*L \in \mathbb{R}^{d \times d}\) be an invertible matrix. There, there exists a permutation \(\sigma: [d] \rightarrow [d]\) such that \(\*L_{i, \sigma(i)} \neq 0\) for all \(i \in [d]\).
\end{lemma}
\begin{proof}
    Since \(\*L\) is invertible, its determinant is nonzero, i.e.,
    \begin{align*}
        \det(\*L) := \sum_{\sigma \in \mathfrak{S}_d} \textnormal{sign}(\sigma) \prod^{d}_{i=1} \*L_{i, \sigma(i)} \neq 0,
    \end{align*}
    where \(\mathfrak{S}_d\) is the set of \(d\)-permutations. This equation implies that at least one term of the sum is nonzero, meaning there exists \(\sigma \in \mathfrak{S}_d\) such that for all \(i \in [d], \*L_{i, \sigma(i)} \neq 0\).
\end{proof}

\subsection{Proof of \cref{th:disent}}\label{app:proof}

Our main proof is inspired by that of \citet[Theorem B.5]{lachapelle_synergies_2023}. While they prove that enforcing a certain notion of sparsity in the context of multi-task prediction problems allows for learning disentangled representations, we can adapt their argument to the setting where we have data from multiple environments, by considering our specific definition of sparsity w.r.t. nonzero variance terms in interventional data (cf. \cref{ssec:ass}).

\setcounter{theorem}{4}
\begin{theorem}[Disentanglement via intervention sparsity.]
    Assume the data generating process described in \cref{ssec:problem}. Let \(\*L \in \mathbb{R}^{d \times m}\) be an injective matrix and \(\hat{\*Z} = \*Z \*L\).
    Suppose \cref{ass:support} holds. Then, if 
    \begin{align*}
        \mathbb{E}_{P_E} \norm{\hat{\*Z}^E} \leq \mathbb{E}_{P_E} \norm{\*Z^E},
    \end{align*}
    \(\hat{\*Z}\) is causally disentangled up to redundancies w.r.t. \(\*Z\) (cf. \cref{def:disent}), where $\norm{\*Z^{\cdot} }$ is defined as the function $\norm{\*Z^{\cdot}}: (\mathcal{E},P_E) \to \mathbb{R}, \ e \mapsto \norm{\*Z^e}$.

    In particular, if \(m=d\), it holds that \(\*L = \*D \*P\), where \(\*D\) is a diagonal invertible matrix and \(\*P\) is a permutation matrix.
\end{theorem}
\begin{proof}
    We begin here with the case where \(m=d\) from which we will then derive the more general case with \(m>d\).

    Since \(\hat{\*Z} = \*Z \*L\) we can write \(\mathbb{E} \norm{\*Z^E \*L} \leq \mathbb{E} \norm{\*Z^E}\), where we drop the subscript of the expectation for brevity.

    We write
    \begin{align*}
        \mathbb{E}\norm{\*Z^E} 
        &= \mathbb{E}_{p(S)} \mathbb{E}_{P_{E \mid S}}\biggr[\sum^d_{j = 1} \ind(\Var(Z^E_j) \neq 0) \mid S \biggr] \\
        &= \mathbb{E}_{p(S)} \sum^d_{j = 1} \mathbb{E}_{P_{E \mid S}}\bigr[ \ind(\Var(Z^E_j) \neq 0) \mid S \bigr] \\
        &= \mathbb{E}_{p(S)} \sum^d_{j = 1} P_{E\mid S}\bigr[\Var(Z^E_j) \neq 0\bigr] \\
        &= \mathbb{E}_{p(S)} \sum^d_{j = 1} \ind(j \in S),
    \end{align*}
    where the last step uses the definition of \(S\).

    Next, we apply a similar treatment to \(\mathbb{E}\norm{\*Z^E \*L'}\) and write
    \begin{align*}
        \mathbb{E}\norm{\*Z^E \*L} 
        &= \mathbb{E}_{p(S)} \mathbb{E}_{P_E}\biggr[\sum^d_{j = 1} \ind(\Var(\*Z^E \*L_{:j}) \neq 0) \mid S \biggr] \\
        &= \mathbb{E}_{p(S)} \sum^d_{j = 1} \mathbb{E}_{P_E}\bigr[ \ind(\Var(\*Z^E \*L_{:j}) \neq 0) \mid S \bigr] \\
        &= \mathbb{E}_{p(S)} \sum^d_{j = 1} P_{E\mid S}\bigr[\Var(\*Z^E \*L_{:j}) \neq 0\bigr] \\
        &= \mathbb{E}_{p(S)} \sum^d_{j = 1} P_{E\mid S}\bigr[\Var(\*Z_S^E \*L_{S,j}) \neq 0\bigr].
    \end{align*}
    Let \(N_j\) denote the nonzero entries of \(\*L\) in column \(j\), i.e., \(N_j := \{i \in [d] \mid \*L_{i,j} \neq 0\}\). Now, we take the intersection of \(S\) and \(N_j\), \(S \cap N_j\), and consider two possible cases. 

    First, assume \(S \cap N_j = \varnothing\). Then, we can see that \(\*L_{S,j} = \*0\) and it follows that 
    \begin{align*}
        P_{E\mid S}\bigr[\Var(\*Z_S^E \*L_{S,j}) = 0\bigr] = 1.
    \end{align*}

    For the second case, consider \(S \cap N_j \neq \varnothing\). Now, we can see that \(\*L_{S,j} \neq \*0\). 
    By \cref{lem:var} we have that
    \begin{align*}
        \Var(\*Z_S^E \*L_{S,j}) \neq 0,
    \end{align*}
    and can therefore directly conclude that
    \begin{align*}
        P_{E\mid S}\bigr[\Var(\*Z_S^E \*L_{S,j}) = 0\bigr] = 0.
      \end{align*}

    Taking both considered cases and recalling that
    \begin{align*}
        P_{E\mid S}\bigr[\Var(\*Z_S^E \*L_{S,j}) \neq 0\bigr] = 1 - P_{E\mid S}\bigr[\Var(\*Z_S^E \*L_{S,j}) = 0\bigr],
    \end{align*}
    we can write
    \begin{align*}
        P_{E\mid S}\bigr[\Var(\*Z_S^E \*L_{S,j}) \neq 0\bigr]
        &= 1 - \ind(S \cap N_j = \varnothing)\\
        &= \ind(S \cap N_j \neq \varnothing),
    \end{align*}
    and consequently
    \begin{align*}
        \mathbb{E}\norm{\hat{\*Z}^E \*L} = \mathbb{E}_{p(S)} \sum^d_{j = 1} \ind(S \cap N_j \neq \varnothing).
    \end{align*}

    Recalling our original constraint, we can now write
    \begin{align}\label{eq:constrain_re}
        \mathbb{E}_{p(S)} \sum^d_{j = 1} \ind(S \cap N_j \neq \varnothing) \leq \mathbb{E}_{p(S)} \sum^d_{j = 1} \ind(j \in S).
    \end{align}
    Since \(\*L\) is an invertible matrix, by \cref{lem:permutation}, we know that there exists a permutation \(\sigma: [d] \rightarrow [d]\), such that \(\*L_{j,\sigma(j)} \neq 0, \forall j \in [d]\). Since the sum over all dimensions \(d\) on the LHS of \cref{eq:constrain_re} is invariant under \(\sigma\), we can apply this permutation to the LHS and by collecting terms write
    \begin{align}\label{eq:constrain_leq_0}
        \mathbb{E}_{p(S)} \sum^d_{j = 1} \ind(S \cap N_{\sigma(j)} \neq \varnothing) - \ind(j \in S) \leq 0.
    \end{align}
    Now, notice \(\forall j \in [d]\)
    \begin{align*}
        \ind(S \cap N_{\sigma(j)} \neq \varnothing) - \ind(j \in S) \geq 0,
    \end{align*}
    which holds since whenever \(j \in S\) it also holds that \(j \in S \cap N_{\sigma(j)}\), since \(\sigma\) by definition permutes the columns of \(\*L\) such that \(\forall j\) there is a nonzero entry \(\*L_{j, \sigma(j)}\), and thus \(j \in N_{\sigma(j)}\).
    
    The overall expression in \cref{eq:constrain_leq_0} is non-positive, while all elements of the sum are non-negative, from which we conclude that all elements of the sum must be equal to zero. We can thus write
    \begin{align*}
        \forall S \in \mathcal{S}, \forall j \in [d], \ind(S \cap N_{\sigma(j)} \neq \varnothing) = \ind(j \in S).
    \end{align*}
    Consequently,
    \begin{align*}
        \forall S \in \mathcal{S}, \forall j \in [d], j \notin S \implies S \cap N_{\sigma(j)} = \varnothing.
    \end{align*}
    Since \(S \cap N_{\sigma(j)} = \varnothing \iff N_{\sigma(j)} \subseteq S^{c}\), where \(S^{c}\) denotes the complement of \(S\), we can write
    \begin{align}
        \forall S \in \mathcal{S}, &\forall j \in [d], j \notin S \implies N_{\sigma(j)} \subseteq S^{c}\\
        &\forall j \in [d], N_{\sigma(j)} \subseteq \bigcap_{S \in \mathcal{S} \mid j \notin S} S^c.\label{eq:N_sub}
    \end{align}
    Recall \cref{ass:support}, which states that
    \begin{align}\label{eq:asspt_cover}
        \bigcup_{S \in \mathcal{S} \mid j \notin S} S = [d]\setminus\{j\}.
    \end{align}
    If we take the complement of both sides of \cref{eq:asspt_cover} and apply De Morgan's law, we get
    \begin{align*}
        \bigcap_{S \in \mathcal{S} \mid j \notin S} S^c = \{j\},
    \end{align*}
    which we insert into \cref{eq:N_sub} and finally see that \(N_{\sigma(j)} = \{j\}\). From this we conclude that \(\*L = \*D \*P\), where \(\*D\) is a diagonal invertible matrix and \(\*P\) is a permutation matrix.

    Now, consider the case where \(m>d\).
    Let \(\*L'\) be an invertible \(d \times d\) submatrix of \(\*L\). Note that if \(\*L'\) is a submatrix of \(\*L\), the original constraint in the theorem implies that \(\mathbb{E} \norm{\*Z^E \*L'} \leq \mathbb{E} \norm{\*Z^E}\). 
    
    For any column \(j\) of \(\*L\) with nonzero entries, there exists a \(d \times d\) invertible submatrix \(\*L'(j)\) that contains the column \(j\) and whose columns are all linearly independent. This holds by the following argument: there is a submatrix \(\*L''\) of \(\*L\) with \(d\) independent columns, since \(\*L\) has rank \(d\). If \(j\) is not already a column of \(\*L''\), by the exchange theorem from linear algebra we can replace some column of \(\*L''\) with \(j\) to obtain the submatrix \(\*L'(j)\) with the desired properties.

    Now, for each such \(\*L'(j)\) we can apply the argument of the proof for the case \(m=d\) and deduce that any (nonzero) column \(j\) of \(\*L\) consists of a single nonzero entry and thus each \(\hat{Z}_j\) must be zero, or a scalar multiple of a latent \(Z_i\).
\end{proof}

The last step of our proof for the case \(m=d\) underlines why we make \cref{ass:support}: after applying De Morgan's law, this assumption leaves us with a set that contains a single element, which in turn directly allows us to conclude that the mixing matrix \(\*L\) consists of a diagonal matrix \(\*D\) and a permutation \(\*P\). Just as \citet{lachapelle_synergies_2023} state, if we did not make \cref{ass:support} and this final step resulted not in a singleton, but a set with multiple elements, this would not allow us to directly conclude that \(\*L = \*D \*P\). If this were the case, we conjecture that we would only be able to achieve partial disentanglement of blocks of variables \citep{lachapelle_partial_2022}.

\section{Experimental Details}
In this section, we present additional details to reproduce all of our experimental results. We begin by outlining the synthetic data generating process and what exactly is subject to randomization across runs, continue with implementation and training details of our proposed method, provide details on the nonlinear SCMs we use in our experiments, introduce the metric we use to quantify our results and finally present additional experimental results not reported in the main text.

\subsection{Synthetic Data Generation}
Across all experimental settings, we use a train/test split of 75/25, meaning if we have \(n=100000\) samples per environment, we use \(n_{\textnormal{train}} = 75000\) samples for training and \(n_{\textnormal{test}} = 25000\) samples to evaluate our metric.

For each random seed, we resample the underlying SCM, resample \(n\) samples from the resulting SCM and reinitialize our algorithm with new initial weights for \(\hat{\*L}\). The ground truth matrix \(\*L\) is sampled once and kept fixed across random seeds. We show that our results do not rely on a cherry picked mixing matrix by reporting additional runs over random seeds for randomly resampled \(\*L\) (cf. \cref{app:ssec:add_results}).

\subsection{Implementation and Training Details}
We use \texttt{pytorch} \citep{paszke_pytorch_2019} to implement our algorithm in practice. For optimization, we use the \texttt{adamW} algorithm \citep{loshchilov_decoupled_2018} with learning rate \(lr = 2 * 10^{-3}\) and otherwise default parameters provided by \texttt{pytorch}. For all experiments, we use a batch size of 4096 and train for 50 epochs.

We refrain from performing hyperparameter optimization in this work. All hyperparameters are chosen to roughly result in equal magnitudes of each respective loss term and shared across all experiments. The values of the hyperparameters are chosen as
\begin{align*}
    \lambda_e & = 1,\\
    \lambda_m & = 1,\\
    \lambda_{\textnormal{diag}} & = 10,\\
    \lambda_{\textnormal{norm}} & = 5.
\end{align*}
For the implementation of ICA, we use the FastICA algorithm \citep{hyvarinen_independent_2000}.

\subsection{Nonlinear SCMs}\label{app:ssec:nonlin}
We sample the adjacency matrix of the graph of both nonlinear SCMs according to the procedure detailed in \cref{ssec:synth_data} with \(p = 0.75\). The resulting graph that is shared by both models is shown in \cref{fig:graph_nonlin}.

The structural equations of SCM 1 are:
\begin{align*}
    Z_1 &:= \eta_1,\\
    Z_2 &:= Z^2_1 + \eta_2,\\
    Z_3 &:= Z^2_1 + Z^2_2 + \eta_3,\\
    Z_4 &:= Z^2_1 + Z^2_2 + Z^2_3 + \eta_4,\\
    Z_5 &:= Z^2_1 + Z^2_3 + Z^2_4 + \eta_5,\\
    Z_6 &:= Z^2_2 + Z^2_3 + Z^2_4 + Z^2_5 + \eta_6,\\
\end{align*}
and those of SCM 2 are defined as:
\begin{align*}
    Z_1 &:= \eta_1,\\
    Z_2 &:= \sin(Z_1) + \eta_2,\\
    Z_3 &:= \sqrt{Z_1 + Z_2} + \eta_3,\\
    Z_4 &:= \log(Z^2_1 + Z_2) + Z^2_3 + \eta_4,\\
    Z_5 &:= Z_3 \cos(Z_1) + \arctan(Z_4) + \eta_5,\\
    Z_6 &:= Z_2 Z_3  e^{\frac{Z^2_4}{Z_5}} + \eta_6.\\
\end{align*}
\begin{figure}[t]
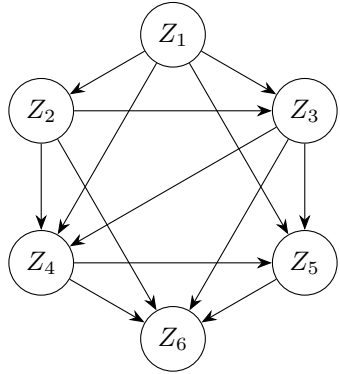

    \centering
    \resizebox{.3\textwidth}{!}{
    \tikz[latent/.append style={minimum size=0.85cm},obs/.append style={minimum size=0.85cm},det/.append style={minimum size=1.15cm}, wrap/.append style={inner sep=2pt}, plate caption/.append style={below left=5pt and 0pt of #1.south east}, on grid]{
        \node[latent](z_1){\(Z_1\)};
        \node[latent, left=2*0.866cm of z_1, yshift=2*-0.5cm](z_2){\(Z_2\)};
        \node[latent, right=2*0.866cm of z_1, yshift=2*-0.5cm](z_3){\(Z_3\)};
        \node[latent, below=2cm of z_2](z_4){\(Z_4\)};
        \node[latent, below=2cm of z_3](z_5){\(Z_5\)};
        \node[latent, below=4cm of z_1](z_6){\(Z_6\)};
        \edge[-{Stealth[length=2mm, width=1.5mm]}]{z_1}{z_2, z_3, z_4, z_5};
        \edge[-{Stealth[length=2mm, width=1.5mm]}]{z_2}{z_3, z_4, z_6};
        \edge[-{Stealth[length=2mm, width=1.5mm]}]{z_3}{z_4, z_5, z_6};
        \edge[-{Stealth[length=2mm, width=1.5mm]}]{z_4}{z_5, z_6};
        \edge[-{Stealth[length=2mm, width=1.5mm]}]{z_5}{z_6}; 
    }
    }
    \caption{Causal graph of SCM 1 and SCM 2.}
    \label{fig:graph_nonlin}
\end{figure}

\subsection{Metric}\label{app:ssec:metric}
Here, we present the formal definition of our used metric, the mean correlation coeffiecient (MCC) \citep{hyvarinen_unsupervised_2016, khemakhem_variational_2020}.

Let \(\*C\) be the Pearson correlation matrix between the ground truth variables \(\*Z\) and the learned variables \(\hat{\*Z}\). Then, the MCC score is defined as
\begin{align*}
    \textnormal{MCC} := \max_{\pi \in \mathfrak{S}_d} \frac{1}{d} \sum^d_{j=1} \lvert \*C_{j, \pi(j)} \rvert,
\end{align*}
where \(\mathfrak{S}_d\) is the set of \(d\)-permutations and \(\lvert \cdot \rvert\) denotes the absolute value.

\subsection{Additional Experiments}\label{app:ssec:add_results}
Here, we present repetitions of the experiment in \cref{ssec:results} where we vary the number of ground truth variables \(d\) under additional random samples of the mixing matrix \(\*L\). The results are reported in \cref{fig:add_exps}.

\begin{figure}[b]
    \centering
    \includegraphics[width=\textwidth]{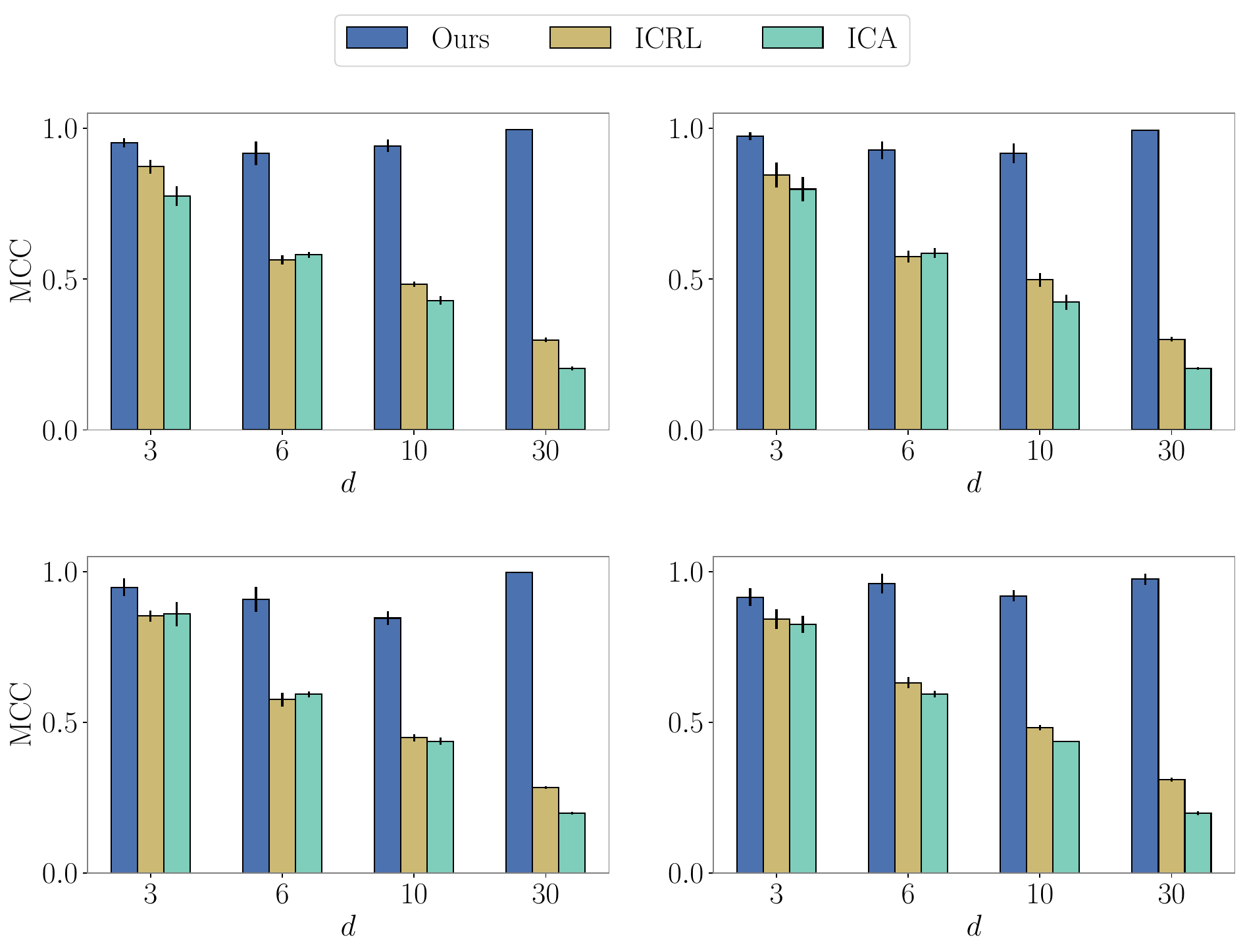}
    \caption{Experiments for changing number of variables \(d\) for additional mixing matrices \(\*L\). Each subfigure corresponds to a different \(\*L\). We report the mean MCC score for five random seeds with error bars indicating the standard error.}
    \label{fig:add_exps}
\end{figure}

\end{document}